\documentclass[conference]{IEEEtran}

\IEEEoverridecommandlockouts

\usepackage{amsmath,amssymb,amsfonts,amstext}
\usepackage{bm}
\usepackage{graphicx}
\usepackage[ruled,linesnumbered,noend]{algorithm2e}
\usepackage{array}
\usepackage{multirow}
\usepackage{booktabs}
\usepackage{url}
\usepackage{hyperref}
\usepackage{lipsum}

\usepackage{caption}
\usepackage{subcaption}
\usepackage[table,xcdraw]{xcolor}

\usepackage{siunitx}
\usepackage{microtype}

\usepackage{amsthm}
\newtheorem{theorem}{Theorem}

\makeatletter
\patchcmd{\@algocf@start}{-1.5em}{0pt}{}{}
\makeatother
\SetArgSty{textup}
\SetKw{Continue}{continue}
\SetKw{Break}{break}

\AtBeginDocument{%
  \setlength{\abovedisplayskip}{4pt plus 2pt minus 1pt}%
  \setlength{\belowdisplayskip}{4pt plus 2pt minus 1pt}%
  \setlength{\abovedisplayshortskip}{0pt plus 2pt}%
  \setlength{\belowdisplayshortskip}{4pt plus 2pt minus 1pt}%
}

\usepackage{xparse}
\let\originalleft\left
\let\originalright\right
\renewcommand{\left}{\mathopen{}\mathclose\bgroup\originalleft}
\renewcommand{\right}{\aftergroup\egroup\originalright}
\NewDocumentCommand\Set{m}{ \left\{#1\right\} }
\NewDocumentCommand\Real{}{ \mathbb{R} }
\NewDocumentCommand\Sym{}{ \mathbb{S} }
\NewDocumentCommand\PDMatrices{m}{\Sym^{#1}_{++}}
\NewDocumentCommand\T{}{\mathsf{T}}
\NewDocumentCommand\Vector{m}{ \boldsymbol{\mathbf{#1}} }
\NewDocumentCommand\Matrix{m}{ \bm{\mathbf{#1}} }
\NewDocumentCommand\Transpose{m}{ \left.{#1}\right.^\T }
\NewDocumentCommand\Inv{m}{{#1}^{-1}}
\NewDocumentCommand\Norm{m}{ \left\Vert#1\right\Vert }
\NewDocumentCommand\Zero{}{ \Matrix{0} }
\NewDocumentCommand\Identity{}{ \Matrix{I} }
\NewDocumentCommand\dd{}{ \mathop{}\!\mathrm{d} }
\NewDocumentCommand\ArgMin{m}{ \operatorname*{argmin}_{#1} }
\newcommand{\R}{\mathcal{R}}
\newcommand{\M}{\mathcal{M}}
\newcommand{\Cfg}{\mathcal{Q}}
\newcommand{\TqM}{\mathcal{T}_{q}\M}
\newcommand{\TqQ}{\mathcal{T}_{q}\Cfg}
\newcommand{\NqQ}{\mathcal{N}_{q}\Cfg}
\newcommand{\D}{\mathrm{D}}
\newcommand{\Metric}{\Matrix{G}}
\newcommand{\MetricAt}[1]{\Matrix{G}_{#1}}
\newcommand{\inner}[3]{\left\langle #1, #2 \right\rangle_{#3}}
\newcommand{\Constraint}{h}
\newcommand{\DConstraint}[1]{\D \Constraint(#1)}
\newcommand{\Dh}{\Matrix{H}}
\newcommand{\Kernel}[1]{\ker #1}
\newcommand{\Image}[1]{\operatorname{im} #1}
\newcommand{\Pseudo}[1]{{#1}^{\dagger}}
\newcommand{\KernelBasis}{\Matrix{B}}
\newcommand{\Chart}{\varphi}
\newcommand{\DChart}[1]{\D \Chart(#1)}
\newcommand{\Dphi}{\Matrix{\Phi}}
\newcommand{\ChartDomain}{\mathcal{U}}
\newcommand{\PullbackMetric}{g}
\newcommand{\PullbackMetricAt}[1]{g(#1)}
\newcommand{\RestrictedMetric}[1]{\Metric_{\Cfg}(#1)}
\newcommand{\qmid}{q_{\text{mid}}}
\newcommand{\Qfree}{\Cfg_{\text{free}}}
\newcommand{\Qobs}{\Cfg_{\text{obs}}}
\newcommand{\Qgoal}{\Cfg_{\text{goal}}}
\newcommand{\qstart}{q_{\text{start}}}

\title{\LARGE \bf Induced Riemannian Metrics for Motion Planning with Constraints}

\author{
Phone Thiha Kyaw$^1{}$, Thomas Cohn$^2{}$, Miguel Angel Rogel Garcia$^1{}$ and Jonathan Kelly$^1{}$
\thanks{$^{1}$Authors are with the Space \& Terrestrial Autonomous Robotic Systems (STARS) Laboratory at the University of Toronto Institute for Aerospace Studies (UTIAS), Toronto, Ontario, Canada, M3H~5T6. Email: {\tt\small <first name>.<last name>@robotics.utias.utoronto.ca}}
\thanks{$^{2}$Authors are with the Robot Locomotion Group (RLG) at the Massachusetts Institute of Technology Computer Science and Artificial Intelligence Laboratory (MIT CSAIL), Cambridge, Massachusetts, United States, 02141. Email: {\tt\footnotesize tcohn@mit.edu}}
}

\begin{document}
\maketitle
\thispagestyle{empty}
\pagestyle{empty}

\begin{abstract}
\looseness=-1
In constrained motion planning problems, task and loop-closure
constraints restrict a robot's motion to a curved, lower-dimensional
submanifold of its configuration space.
Planners measure path length with a metric, which sets the cost of
moving in each direction.
Under the Euclidean metric, this cost is the same everywhere, whereas
under a general Riemannian metric, such as the kinetic-energy metric,
the cost can vary with direction and configuration.
Existing methods often describe the submanifold either implicitly, as
a constraint level set, or explicitly, through a parameterization.
The implicit representation is typically combined with the Euclidean
metric of the configuration space, and the explicit representation
with the parameter domain, so the path length that a planner
minimizes depends on the representation.
Instead, we measure path length with the induced metric, which the
submanifold inherits from a Riemannian metric on the configuration
space.
The implicit and explicit representations yield the same induced
metric, expressed in different coordinates, and hence the same
geometry.
This result holds for any Riemannian metric on the configuration space,
not only the Euclidean one.
The choice of metric is therefore independent of the choice of
representation.
Using this result, we extend planning under a Riemannian metric from
unconstrained spaces to constraint submanifolds by applying the induced
metric in both a sampling-based planner and a trajectory optimizer.
For an explicit representation, the induced metric also accounts for
the distortion that the parameterization introduces.
In experiments on a bimanual manipulation setup with two Franka arms
under end-effector task constraints, we compare the Euclidean and
kinetic-energy metrics.
\end{abstract}

\section{Introduction}
\label{sec:introduction}

Robotic motion planning is often posed as a search for collision-free paths through a configuration space.
Many practical tasks impose additional equality constraints on this search.
For example, holding a tool level, keeping an end effector on a surface, or closing a kinematic loop between two arms restricts the robot to a lower-dimensional subset of feasible configurations.
Such constraints define a smooth submanifold of the configuration space, and a feasible motion must remain on this submanifold throughout~\cite{kingston2018sampling,cohn2024constrained}.
Because the submanifold is generally curved, the cost of moving between two configurations is governed by its intrinsic geometry rather than by distances in the ambient space.

\begin{figure}[!t]
\centering
\includegraphics[width=\columnwidth]{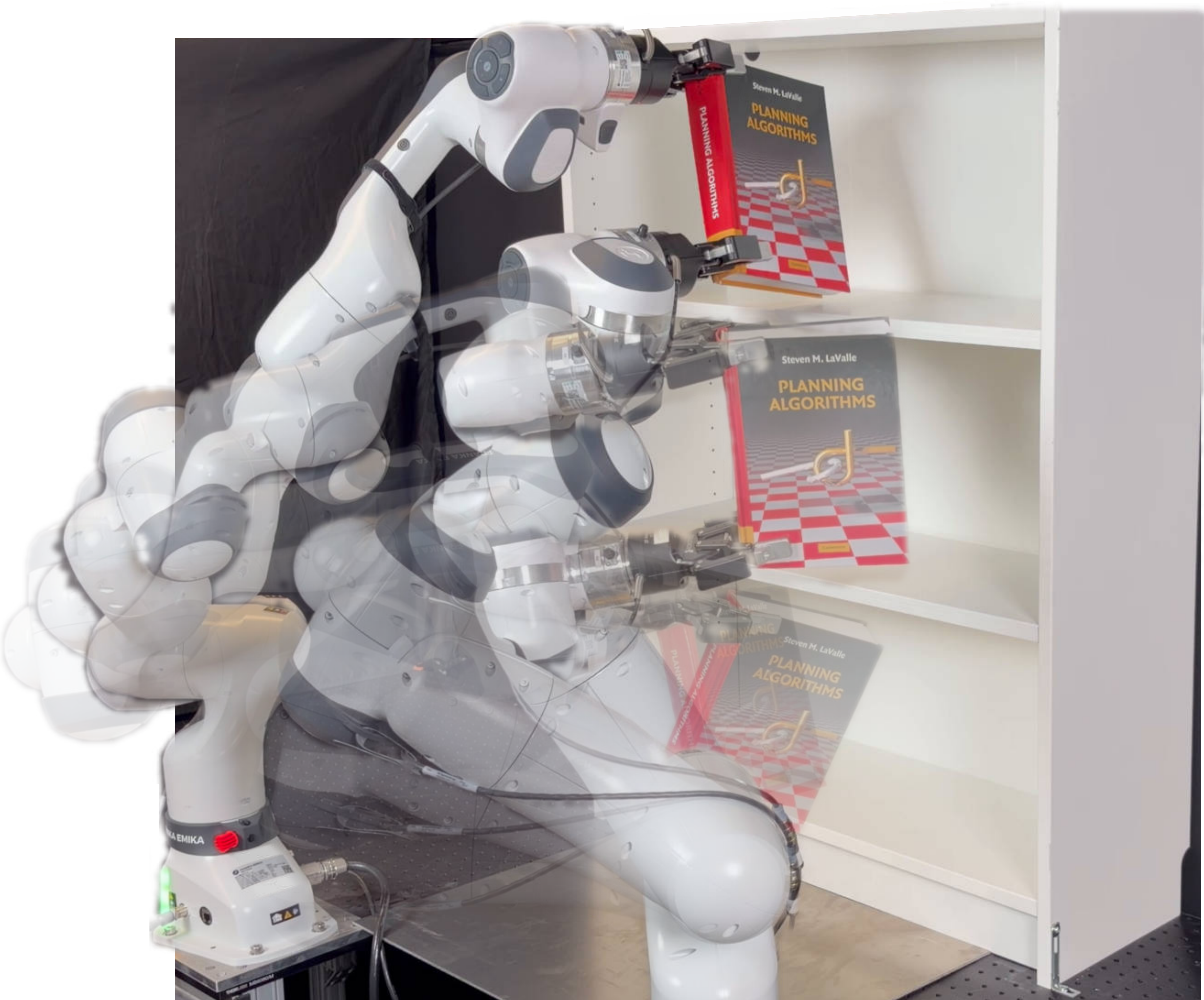}
\caption{
Two Franka arms executing a constraint-satisfying motion plan under the kinetic-energy Riemannian metric to transfer a book between shelf compartments in our bimanual experimental setup.
}
\label{fig:teaser}
\vspace{-5mm}
\end{figure}

\looseness=-1
The cost of moving between configurations also depends on the chosen  optimization objective.
For example, rotating the shoulder joints of a manipulator through a
given angle takes more energy than rotating the wrist through the same
angle, because the shoulder joints move most of the arm.
The path that is shortest in joint space, measured with the Euclidean
metric, is therefore not necessarily the path that uses the least
energy.
However, a shortest-path planner can still minimize a different cost
if length is measured with a metric under which shorter paths have
lower cost.
A Riemannian metric allows us to measure length in this way.
It assigns every configuration an inner product that varies smoothly across the space, and this inner product sets the cost of moving in each direction~\cite{lee2018introduction}.
Many planners and controllers already use such metrics built for the task, with the mass matrix turning joint-space length into kinetic energy~\cite{bullo2005geometric,jaquier2022riemannian}, and with other metrics encoding collision safety or the behaviour of a controller~\cite{ratliff2018riemannian,cheng2021rmpflow,klein2023design}.
Recent work described how sampling-based planners can plan through Euclidean configuration spaces with arbitrary Riemannian metrics~\cite{kyaw2026geometry,kyaw2026direct}.
We extend those works to constraint submanifolds, where feasible motion is measured by the metric that the submanifold inherits from the configuration space.

\looseness=-1
Constraint submanifolds are represented by constrained motion planners either implicitly, as the zero level set of a constraint map~\cite{berenson2009manipulation,jaillet2012path,kim2016tangent,kingston2018sampling}, or explicitly, through a parameterization that satisfies the constraint by construction~\cite{cohn2024constrained,cohn2026planning}. 
Kingston et al. unify the planners built on the implicit representation by decoupling the choice of planner from the method used to satisfy the constraint, though the resulting planners often assume a Euclidean metric on the ambient space~\cite{kingston2019exploring}.
Garg et al. observe that a nonlinear parameterization distorts distances, and use nonconvex optimization to locally correct for the distortion, without identifying the correction as the metric that the submanifold induces~\cite{garg2025planning}.
However, neither method explicitly reasoned about the underlying geometry of the constraint manifold or the impact of non-Euclidean ambient metrics.

In this work, we derive the geometry that the constraint submanifold inherits for both the implicit and explicit representations, and prove that both yield the same induced metric for any ambient Riemannian metric.
The two therefore result in the same lengths, distances, and geodesics on the submanifold, under both constant and smoothly varying ambient metrics.
We then apply this induced metric inside a sampling-based planner and a trajectory optimizer, so that both minimize the same Riemannian length on the constraint submanifold.

Our main contributions are as follows.
\begin{itemize}
\item We prove that the implicit and explicit representations of a constraint manifold induce the same metric, related by an invertible change of basis, for any Riemannian metric on the ambient space.
\item We show how this induced metric is used by a sampling-based planner and by a trajectory optimizer on the constraint manifold, so that both solvers minimize the same Riemannian length functional.
\item We demonstrate our approach in simulation and on a real bimanual setup with two Franka arms under a loop-closure constraint, comparing the Euclidean and the kinetic-energy metric in both representations.
\end{itemize}
\section{Related Work}
\label{sec:related_work}

\looseness=-1
Planning under equality constraints is a search over a constraint manifold, a lower-dimensional subset of the configuration space on which the constraint holds~\cite{kingston2018sampling}.
Because this manifold has measure zero in the original space, uniform sampling almost surely returns infeasible configurations, yet a sampling-based planner has to generate both its samples and motions on the manifold itself.
Projection-based methods restore feasibility with a numerical procedure (e.g. Jacobian pseudoinverse)~\cite{berenson2009manipulation,berenson2011task,stilman2010global,iyer2026vectorizing}, whereas continuation-based methods cover the manifold with local piecewise-linear approximations of the tangent space and grow the tree within this atlas of charts~\cite{jaillet2012path,kim2016tangent,bordalba2020randomized}.
Kingston et al. unify projection and continuation by decoupling the choice of planner from the method used to satisfy the constraint, while preserving probabilistic completeness and asymptotic optimality~\cite{kingston2019exploring}.
All of them describe the constraint manifold implicitly, and their sampling, local planning, and cost evaluation subroutines often measure motion with the Euclidean metric of the ambient configuration space.

A second approach parameterizes the constraint manifold directly, so that the valid configurations have positive measure in the parameter domain and a planner can sample them without requiring any projection~\cite{mirabel2018handling,cohn2024constrained}.
Inverse kinematics has long been used to parameterize constraint manifolds~\cite{han2001kinematics,cortes2005sampling}.
More recently, this approach enabled optimization-based planning algorithms to use parameterizations, with gradients from automatic differentiation~\cite{cohn2024constrained} or the inverse function theorem~\cite{cohn2026planning}.
Optimization over these parameterizations has been developed for graphs of convex sets, which recovers
exact geodesics by restricting to flat manifolds where each chart is a local isometry~\cite{cohn2025non}.
Garg et al. observe that a nonlinear parameterization distorts distances, so that a convex objective on the parameter domain yields paths that are suboptimal in the original space, and they correct the objective accordingly~\cite{garg2025planning}.
In contrast, we identify this correction as the pullback of the ambient metric, and prove that it agrees with the restriction of that metric to the tangent spaces of the constraint manifold.

\looseness=-1
A different class of methods starts from the metric that determines the geometry of the ambient space rather than from the constraint that defines the manifold.
Here, treating the configuration space as a Riemannian manifold makes the cost of a motion a property of that geometry rather than of the chosen coordinates.
The mass matrix, for example, defines a kinetic-energy Riemannian metric that measures a motion by the effort needed to execute it~\cite{bullo2005geometric,jaquier2022riemannian}, while other methods fuse several motion policies by combining the metrics that define them~\cite{ratliff2018riemannian,cheng2021rmpflow}, or reshape a metric so that shortest paths stay away from obstacles and singularities~\cite{klein2023design}.
Sampling-based planners have been extended to such general non-Euclidean metrics~\cite{lukyanenko2023probabilistic}, and recent work plans geodesics and samples informed sets under configuration-dependent metrics on unconstrained manifolds~\cite{kyaw2026geometry,kyaw2026direct}.
Each of these metrics is defined on a manifold that the robot moves on freely, whereas we take the ambient metric as given and work with the geometry that a constraint manifold inherits from it.

Trajectory optimization methods on constraint manifolds either penalize constraint violation or eliminate the constraint through a representation of the manifold~\cite{ratliff2009chomp,kalakrishnan2011stomp,schulman2014motion,toussaint2014newton}.
Bonalli et al. instead lift the problem into an equivalent one defined over a space enjoying a Euclidean structure, and handle both the implicit and the explicit representation~\cite{bonalli2019trajectory}.
Each of these optimizers minimizes a Euclidean length or smoothness functional on a chart or parameter domain, which differs from the length that the constraint manifold induces.
In contrast, we minimize this induced length directly, using the same metric in both a sampling-based planner and a trajectory optimizer.
\section{Preliminaries}
\label{sec:preliminaries}

\looseness=-1
This section reviews the geometric tools we use in the rest of the paper and states the constrained planning problem.
We keep the treatment brief and refer the reader to the standard texts for the full theory~\cite{lee2018introduction}.

\subsection{Riemannian Manifolds and Metrics}
\label{subsec:riemannian}

\looseness=-1
Let $\M$ be an $n$-dimensional configuration manifold of the robot, embedded in a linear space (we assume $\M$ is an open subset of $\Real^{n}$).
The tangent space $\TqM$ at a point $q \in \M$ collects the velocity vectors of all smooth curves on $\M$ that pass through $q$.
Every configuration has its own tangent space, and the collection of all such tangent spaces forms the tangent bundle $\mathcal{T}\M = \Set{ (q, v) : q \in \M, \, v \in \TqM }$.
We work in ambient coordinates throughout, so that a tangent vector can be identified as a vector in $\Real^{n}$.
A Riemannian metric $\MetricAt{q}$ on $\M$ assigns each of these tangent spaces a smoothly varying inner product
\begin{equation}
\label{eqn:inner-product}
\inner{u}{v}{q} = \Transpose{u}\,\MetricAt{q}\,v, \qquad \MetricAt{q} \in \PDMatrices{n},
\end{equation}
where $\MetricAt{q}$ is symmetric positive definite and induces the norm $\Norm{v}_{q} = \inner{v}{v}{q}^{1/2}$.
The metric induces a notion of length for curves and hence a distance function, thereby turning $\M$ into a metric space.
For a piecewise smooth curve $\gamma : [0,1] \to \M$, its length is
\begin{equation}
\label{eqn:arc-length}
L[\gamma] = \int_{0}^{1} \Norm{\dot\gamma(t)}_{\gamma(t)} \, \dd t
= \int_{0}^{1} \sqrt{ \Transpose{\dot\gamma(t)}\,\MetricAt{\gamma(t)}\,\dot\gamma(t) } \, \dd t,
\end{equation}
and the infimum of this length over all curves joining $q_x$ to $q_y$ is the \textit{Riemannian distance} $d(q_x, q_y)$.
Geodesics are the curves that locally minimize this distance, the curved-space counterpart of straight lines.
Between fixed endpoints, they are equivalently the critical curves of the energy functional
\begin{equation}
\label{eqn:energy}
E[\gamma] = \frac{1}{2} \int_{0}^{1} \Transpose{\dot\gamma(t)}\,\MetricAt{\gamma(t)}\,\dot\gamma(t) \, \dd t .
\end{equation}
Note that the length functional in~\eqref{eqn:arc-length} is invariant under reparameterization.
More importantly for our purposes, it is also invariant under a change of coordinates, since the metric transforms as a $(0,2)$-tensor.
We will exploit precisely this transformation law in Section~\ref{subsec:equivalence}.

\subsection{Constraint Submanifolds}
\label{subsec:constraint-submanifolds}

We encode task and loop-closure constraints with a smooth constraint map $\Constraint : \M \to \Real^{k}$, where $k < n$.
The configurations that satisfy the constraint form its zero level set
\begin{equation}
\label{eqn:submanifold}
\Cfg = \Set{ q \in \M : \Constraint(q) = \Vector{0} }.
\end{equation}
We assume throughout that the constraint Jacobian $\DConstraint{q} \in \Real^{k \times n}$ has full row rank everywhere on $\Cfg$, so that $\Vector{0}$ is a regular value of $\Constraint$.
Configurations where this rank condition fails are singular, and we exclude them from $\mathcal{Q}$.
The set $\Cfg$ is then a smooth embedded submanifold of $\M$ of dimension $d = n - k$, and its tangent space at $q$ is the kernel of the constraint Jacobian~\cite{lee2012smooth},
\begin{equation}
\label{eqn:tangent-kernel}
\TqQ = \Kernel{\DConstraint{q}} = \Set{ v \in \TqM : \DConstraint{q}\,v = \Vector{0} }.
\end{equation}
The feasible velocities are exactly the velocities that keep the constraint satisfied to first order, since differentiating $\Constraint(\gamma(t)) \equiv \Vector{0}$ along a curve in $\Cfg$ gives $\DConstraint{q}\,\dot\gamma = \Vector{0}$.
Restricting the ambient metric $\Metric$ of $\Cfg$ to the tangent spaces $\TqQ$ at every $q \in \Cfg$ makes $\Cfg$ a Riemannian manifold in its own right, and the resulting induced metric determines the cost of feasible motion.
We construct this metric in Section~\ref{sec:method}.

\subsection{Problem Formulation}
\label{subsec:problem}

Let $\Qobs \subsetneq \Cfg$ be the configurations in collision with obstacles, and let $\Qfree = \Cfg \setminus \Qobs$ be the collision-free part of the configuration manifold.
Given a start configuration $q_{\text{start}} \in \Qfree$ and a set of goal configurations $\Qgoal \subseteq \Qfree$, we seek a feasible path that minimizes the Riemannian length under the induced metric,
\begin{equation}
\label{eqn:problem}
\begin{split}
\gamma^{*} = \ArgMin{\gamma \in \Sigma} \Big\{ L[\gamma] \;\Big|\;
& \gamma(0) = \qstart,\ \gamma(1) \in \Qgoal, \\
& \gamma(t) \in \Qfree,\ \forall t \in [0,1] \Big\},
\end{split}
\end{equation}
where $\Sigma$ is the set of piecewise smooth paths.
We assume throughout that this minimum is attained, which holds whenever $\Qfree$ is compact.
Because $\Cfg$ is generally curved, the solution of~\eqref{eqn:problem} is not a straight line in the ambient coordinates of $\M$.
A planner that measures distance in those coordinates therefore optimizes the wrong quantity.
Its subroutines must follow the induced geometry of $\Cfg$ instead, which we develop next.
\section{Induced Geometry of Constraint Submanifolds}
\label{sec:method}

\looseness=-1
This section constructs the metric that the constraint submanifold $\Cfg$ inherits from the configuration manifold $\M$.
We derive this metric twice, first from the implicit constraint level set (Section~\ref{subsec:implicit}) and then from an explicit parameterization (Section~\ref{subsec:explicit}).
We then prove that the two representations are equivalent for any ambient Riemannian metric (Section~\ref{subsec:equivalence}).

\subsection{Implicit Representation}
\label{subsec:implicit}

\looseness=-1
In the implicit representation, we describe $\Cfg$ as the zero level set~\eqref{eqn:submanifold} of the constraint map and obtain its geometry from the constraint Jacobian $\Dh = \DConstraint{q} \in \Real^{k \times n}$.
From~\eqref{eqn:tangent-kernel}, the tangent space is the kernel $\TqQ = \Kernel{\Dh}$, and its orthogonal complement with respect to the Euclidean inner product is the normal space $\NqQ = \Image{\Transpose{\Dh}}$.
We use the normal space to return to $\Cfg$ and the tangent space to move along it.
Since operations in the ambient space generally do not satisfy the constraint, we project them back onto $\Cfg$, as is commonly done in constrained sampling-based planning~\cite{berenson2009manipulation,kingston2018sampling}.
Linearizing the constraint about $q$ and taking the correction of minimum Euclidean norm, which lies in $\NqQ$, gives the update
\begin{equation}
\label{eqn:projection}
q \leftarrow q - \Pseudo{\Dh}\,\Constraint(q), \qquad
\Pseudo{\Dh} = \Transpose{\Dh}\bigl( \Dh\,\Transpose{\Dh} + \lambda \Identity \bigr)^{-1}.
\end{equation}
Here, $\Pseudo{\Dh}$ is the damped pseudoinverse and $\lambda \ge 0$ is a Tikhonov regularization that keeps it well conditioned when $\Dh$ loses rank.
Iterating~\eqref{eqn:projection} to convergence defines the projection operator $\Pi_{\Cfg}$, which maps an ambient configuration to a nearby one on $\Cfg$.\footnote{The correction in~\eqref{eqn:projection} is normal to $\TqQ$ in the Euclidean sense, not in the metric $\MetricAt{q}$, whose normal space is $\Image{\Inv{\MetricAt{q}}\Transpose{\Dh}}$. The retraction~\eqref{eqn:implicit-retraction} is therefore second order under the Euclidean metric and first order in general~\cite{absil2012projection}.}
To move along $\Cfg$, we take a tangent step in the ambient space and project the result back, which gives a valid retraction on $\Cfg$,
\begin{equation}
\label{eqn:implicit-retraction}
\R_{q} : \TqQ \to \Cfg, \qquad \R_{q}(v) = \Pi_{\Cfg}(q + v),
\end{equation}
satisfying $\R_{q}(0) = q$ and $\D\R_{q}(0)[v] = v$~\cite{boumal2023introduction}.

\looseness=-1
We now restrict the ambient metric $\MetricAt{q}$ to the tangent space, turning $\Cfg$ into a Riemannian manifold in its own right.
Let $\KernelBasis \in \Real^{n \times d}$ be an orthonormal basis of $\Kernel{\Dh}$, that is, a matrix whose $d$ columns span $\TqQ$ and satisfy $\Transpose{\KernelBasis}\KernelBasis = \Identity_{d}$, obtained in practice from a singular value decomposition of $\Dh$.
A tangent vector is expressed in this basis as $v = \KernelBasis\,\xi$ with reduced coordinates $\xi \in \Real^{d}$, and its squared length under the ambient metric is $\Transpose{\xi}\bigl(\Transpose{\KernelBasis}\,\MetricAt{q}\,\KernelBasis\bigr)\xi$.
Therefore the induced metric on $\Cfg$, expressed in the orthonormal kernel basis, is the $d \times d$ symmetric positive definite matrix
\begin{equation}
\label{eqn:restricted-metric}
\RestrictedMetric{q} = \Transpose{\KernelBasis}\,\MetricAt{q}\,\KernelBasis \in \PDMatrices{d}.
\end{equation}
Positive definiteness follows from that of $\MetricAt{q}$, and the induced metric varies over $\Cfg$ through both $\MetricAt{q}$ and the basis $\KernelBasis$, which itself changes with $q$.

\subsection{Explicit Parameterization}
\label{subsec:explicit}

\looseness=-1
In the explicit approach, we describe $\Cfg$ locally as the image of a parameterization\footnote{Following~\cite{lee2012smooth,cohn2026planning}, we call $\Chart$ a parameterization and reserve \emph{chart} for its inverse, which maps $\Cfg$ into $\Real^{d}$; much of the constrained planning literature instead calls a local parameterization of this kind a chart~\cite{jaillet2012path,kingston2019exploring}.} $\Chart : \ChartDomain \subseteq \Real^{d} \to \Cfg$ and obtain its geometry from the differential $\Dphi = \DChart{u} \in \Real^{n \times d}$.
This parameterization satisfies the constraint by construction, $\Constraint(\Chart(u)) \equiv \Vector{0}$, so every parameter $u$ in the open domain $\ChartDomain$ maps to a configuration on $\Cfg$.
Unlike in the implicit representation, where the feasible set has measure zero in $\M$, we can directly sample $\ChartDomain$ to generate feasible configurations without requiring any projection.
Differentiating $\Constraint(\Chart(u)) \equiv \Vector{0}$ gives $\Dh\,\Dphi = \Zero$, so the columns of $\Dphi$ lie in $\Kernel{\Dh}$.
Since $\Dphi$ has full column rank, $\Image{\Dphi}$ and $\Kernel{\Dh}$ both have dimension $d$, so the two subspaces coincide,
\begin{equation}
\label{eqn:chart-tangency}
\Image{\Dphi} = \Kernel{\Dh} = \TqQ, \qquad q = \Chart(u).
\end{equation}

\looseness=-1
We now pull the ambient metric $\MetricAt{q}$ back to the parameter domain.
The differential pushes a 
parameter velocity $\dot u \in \mathcal{T}_{u}\ChartDomain = \Real^{d}$ forward to the tangent vector $\Dphi\,\dot u \in \TqQ$, whose squared length under the ambient metric is $\Transpose{\dot u}\bigl(\Transpose{\Dphi}\,\MetricAt{q}\,\Dphi\bigr)\dot u$.
Factoring out $\dot u$, the induced metric on $\Cfg$, expressed in parameter coordinates, is the $d \times d$ symmetric positive definite matrix
\begin{equation}
\label{eqn:pullback-metric}
\PullbackMetricAt{u} = \Transpose{\Dphi}\,\MetricAt{q}\,\Dphi \in \PDMatrices{d}, \qquad q = \Chart(u).
\end{equation}
Curves in the parameter domain measured under $\PullbackMetric$ therefore have exactly their true length on $\Cfg$, since a parameter curve $c$ and its image $\gamma = \Chart \circ c$ have velocities related by $\dot\gamma = \Dphi\,\dot c$ and hence equal speeds under the two metrics.
Treating the parameter domain as flat instead discards~\eqref{eqn:pullback-metric} and distorts distances~\cite{garg2025planning}.
In practice, the parameterization itself may be available in closed form or, for an articulated robot, may come from an analytic inverse-kinematics solution~\cite{cohn2024constrained}, whose differential follows from the forward kinematics by the inverse function theorem~\cite{cohn2026planning}.
We use such a parameterization for the bimanual task in Section~\ref{sec:experiments}.

\subsection{Equivalence of the Two Representations}
\label{subsec:equivalence}

The two constructions above write one metric in two different bases.
By~\eqref{eqn:chart-tangency}, the columns of the differential $\Dphi$ and the columns of the kernel basis $\KernelBasis$ both span the same $d$-dimensional subspace $\TqQ$.
Since any two bases of one subspace are related by an invertible change of basis, we have
\begin{equation}
\label{eqn:change-of-basis}
\Matrix{R} = \Transpose{\KernelBasis}\,\Dphi \in \Real^{d \times d}, \qquad \Dphi = \KernelBasis\,\Matrix{R},
\end{equation}
where the formula for $\Matrix{R}$ follows by multiplying $\Dphi = \KernelBasis\,\Matrix{R}$ on the left by $\Transpose{\KernelBasis}$ and using $\Transpose{\KernelBasis}\KernelBasis = \Identity_{d}$.
The following result shows that this one matrix relates the two induced metrics.

\begin{theorem}[Equivalence of induced metrics]
\label{thm:equivalence}
Let $q\!=\!\Chart(u)\!\in\!\Cfg$, let $\KernelBasis$ be an orthonormal basis of $\Kernel{\Dh}$, and let $\Matrix{R} = \Transpose{\KernelBasis}\,\Dphi$.
Then $\Matrix{R}$ is invertible and the pullback metric~\eqref{eqn:pullback-metric} and the induced metric~\eqref{eqn:restricted-metric} are equivalent,
\begin{equation}
\label{eqn:equivalence}
\PullbackMetricAt{u} = \Transpose{\Matrix{R}}\RestrictedMetric{q}\,\Matrix{R}.
\end{equation}
The two therefore define the same inner product on $\TqQ$, and hence the same lengths, angles, and geodesics on $\Cfg$.
\end{theorem}
\begin{proof}
By~\eqref{eqn:chart-tangency}, the columns of $\Dphi$ and the columns of $\KernelBasis$ span the same subspace $\TqQ$, and both matrices have full column rank $d$.
Since $\Dphi = \KernelBasis\,\Matrix{R}$ has rank $d$, the $d \times d$ matrix $\Matrix{R}$ has rank $d$ as well and is therefore invertible.
Substituting $\Dphi = \KernelBasis\,\Matrix{R}$ into the pullback metric~\eqref{eqn:pullback-metric} gives
\begin{equation*}
\PullbackMetricAt{u}
= \Transpose{\Dphi}\,\MetricAt{q}\,\Dphi
= \Transpose{\Matrix{R}}\bigl( \Transpose{\KernelBasis}\,\MetricAt{q}\,\KernelBasis \bigr)\Matrix{R}
= \Transpose{\Matrix{R}}\,\RestrictedMetric{q}\,\Matrix{R}.
\end{equation*}
Let $v \in \TqQ$ be a tangent vector with coordinates $\dot u \in \Real^{d}$ in the basis $\Dphi$ and coordinates $\xi \in \Real^{d}$ in the basis $\KernelBasis$, so that $v = \Dphi\,\dot u = \KernelBasis\,\xi$.
Since $\KernelBasis$ has full column rank, $\KernelBasis\,\xi = \KernelBasis\,\Matrix{R}\,\dot u$ implies $\xi = \Matrix{R}\,\dot u$, and the squared length of $v$ is therefore the same in either basis,
\begin{equation*}
\Transpose{\dot u}\PullbackMetricAt{u}\dot u
= \Transpose{(\Matrix{R}\dot u)}\RestrictedMetric{q}(\Matrix{R}\dot u)
= \Transpose{\xi}\RestrictedMetric{q}\,\xi
= \Transpose{v}\MetricAt{q}\,v .
\end{equation*}
The two metrics give every tangent vector the same length, and since an inner product is determined by the squared norm it induces, they define the same inner product on $\TqQ$.
\end{proof}

\looseness=-1
The relation in~\eqref{eqn:equivalence} is precisely the $(0,2)$-tensor transformation law of Section~\ref{subsec:riemannian}, applied to the change of basis in~\eqref{eqn:change-of-basis}.
The restricted metric $\RestrictedMetric{q}$ and the pullback metric $\PullbackMetricAt{u}$ are therefore the same tensor on the constraint submanifold, written by the implicit representation in an orthonormal basis of $\Kernel{\Dh}$ and by the explicit representation in the basis given by the columns of $\Dphi$.
Since the ambient metric $\MetricAt{q}$ enters~\eqref{eqn:restricted-metric} and~\eqref{eqn:pullback-metric} in the same place, Theorem~\ref{thm:equivalence} holds for \emph{any} positive definite ambient metric.
Existing constrained planners fix a constant ambient metric in advance and measure motion with it throughout the search~\cite{jaillet2012path,kim2016tangent,kingston2019exploring,cohn2024constrained}.
Our construction generalizes these approaches by supporting both constant and smoothly varying ambient metrics, so that the induced geometry captures costs that vary smoothly with configuration.

\section{Planning with the Induced Metric}
\label{sec:planning}

We now present how existing motion planning algorithms can make use of the induced metric without requiring any change to the search strategies on which they are built.
We first describe the subroutines of sampling-based planners that depend on the intrinsic geometry of the constraint submanifold, and show how each of them is realized under the induced metric (Section~\ref{subsec:planning-sbp}).
We then describe the corresponding treatment for trajectory optimization (Section~\ref{subsec:planning-trajopt}).

\subsection{Sampling-Based Planning}
\label{subsec:planning-sbp}

\looseness=-1
Sampling-based motion planning algorithms construct a tree or a graph by repeatedly sampling random configurations,
connecting nearby ones through a local planning procedure,
and comparing the resulting candidate paths according to a cost function~\cite{karaman2011sampling}.
All three operations depend on the geometry of the space in which the planner searches, and on a constraint submanifold, each of them must therefore respect the intrinsic geometry that the submanifold inherits from the ambient configuration space~\cite{kingston2018sampling}.
Since an asymptotically optimal planner converges to the optimal path with respect to the cost function that it optimizes, the planner recovers the solution of~\eqref{eqn:problem} precisely when this cost function is the induced Riemannian length.

\textit{Sampling.}
Uniform samples generated in the ambient configuration space lie on the constraint submanifold with probability zero, so a planner that relies on the implicit representation cannot obtain feasible configurations by sampling alone.
We therefore project the sample back onto $\Cfg$ with $\Pi_{\Cfg}$, and reject it whenever the Newton iteration in~\eqref{eqn:projection} fails to converge~\cite{berenson2009manipulation,stilman2010global}.
This projected sample comes either from the ambient manifold, as in projection-based methods, or from the tangent space at a configuration already on $\Cfg$, as in continuation-based methods~\cite{jaillet2012path,kim2016tangent}.
The explicit representation avoids these steps, since sampling the parameter domain $\ChartDomain$ already produces configurations that satisfy the constraint by construction.

\looseness=-1
\textit{Local planning.}
Connecting two configurations on the constraint submanifold requires a curve that remains feasible along its entire length and follows the induced geometry rather than the geometry of the ambient manifold.
Existing constrained planners traverse the submanifold using local parameterizations of the tangent space~\cite{jaillet2012path,kim2016tangent,kingston2019exploring}, but they measure the resulting motion with a constant ambient metric and therefore do not account for the smooth variation of the induced metric across the submanifold.
We keep the curve that each representation already provides and change how we measure it.
The explicit parameterization approach connects two parameters with a straight chord in $\ChartDomain$, which maps to a feasible curve on $\Cfg$ by construction, and the implicit representation projects the ambient chord back onto the level set with $\Pi_{\Cfg}$.
Both curves stay on $\Cfg$ along their entire length, so the induced metric enters through the cost of an edge only, which we describe next.
A local planner could instead trace a discrete geodesic of the induced metric by stepping along the Riemannian gradient of a squared distance potential~\cite{kyaw2026geometry}.
Such a curve follows the geometry more closely, but it evaluates the induced metric once per step rather than once per edge, and we keep the cheaper connection so that an anytime planner can spend its budget more on generating samples and rewiring tree edges.

\textit{Cost evaluation.}
Nearest-neighbour queries, edge costs, and rewiring operations all require the induced distance between two configurations.
Computing it exactly, however, requires solving a geodesic boundary value problem, which is impractical within the inner loop of a planner.
We therefore adopt the midpoint approximation of~\cite{kyaw2026geometry},
\begin{equation}
\label{eqn:midpoint-distance}
\begin{split}
\qmid &= \R_{q_x}\bigl( \tfrac{1}{2}\,\R^{-1}_{q_x}(q_y) \bigr), \\
\hat{c}(q_x, q_y) &= \Norm{ \R^{-1}_{\qmid}(q_y) - \R^{-1}_{\qmid}(q_x) }_{\qmid},
\end{split}
\end{equation}
where $\R^{-1}_{q}$ is the inverse retraction, defined for configurations close enough to $q$, and $\Norm{\cdot}_{\qmid}$ is the norm of the induced metric $\RestrictedMetric{\qmid}$.
In the implicit representation, we evaluate $\R^{-1}_{q}$ at a nearby configuration $q'$ by projecting the ambient difference onto the tangent space, $v \leftarrow (\Identity - \Pseudo{\Dh}\Dh)(q' - q)$, and then repeating the correction $v \leftarrow v + (\Identity - \Pseudo{\Dh}\Dh)\bigl(q' - \R_{q}(v)\bigr)$ until $\R_{q}(v)$ reaches $q'$.
In the explicit representation, the inverse retraction is the difference $u' - u$ of the two parameters.
The approximation evaluates the induced metric once, at the retraction midpoint $\qmid$, and agrees with the true induced distance to third order in the separation between the two configurations.
Measuring an edge in this way, rather than treating the ambient manifold or the parameter domain as flat, ensures that the planner minimizes the induced Riemannian length in~\eqref{eqn:problem}.
Theorem~\ref{thm:equivalence} further guarantees that this cost is identical under the two representations, so the two pose the
same optimization problem.
An asymptotically optimal planner therefore converges to the same optimal cost in either representation, while the particular path it returns depends on its random samples and on whether the optimal path is unique.

\subsection{Trajectory Optimization}
\label{subsec:planning-trajopt}

In contrast to sampling-based motion planning algorithms, trajectory optimization approaches use decision variables in a mathematical program to represent the robot's trajectory~\cite[\S 6.2]{russtedrake2024manipulation}.
An objective function (e.g. path length) and various constraints (e.g. collision avoidance) are formulated in terms of the decision variables, and then an optimization algorithm is used to find a solution.
Trajectory optimization inherently provides extensive modeling freedom, and can be applied directly to constrained motion planning by simply imposing the relevant constraint along the trajectory~\cite{bonalli2019trajectory}.
But this modeling freedom comes with a lack of guarantees: most constraints of interest in robotics yield a nonconvex feasible set, making it challenging to obtain feasible solutions.
In many cases, the success of an optimization problem is heavily dependent on the initial guess, and the use of a (suboptimal) feasible path from a sampling-based planner is a common technique.

\looseness=-1
In this paper, we implement trajectory optimization for the explicit representation; the trajectory is represented in the parameterization's minimal coordinates, and is lifted to configuration space in order to apply costs and constraints.
We use the Drake's KinematicTrajectoryOptimization implementation
, which uses a B-spline representation of the path $\gamma(s)=\sum_{i=0}^N B_i(s)P_i$, where $B_i$ is the $i$th basis function and $P_i$ is the $i$th control point.
We impose collision avoidance constraints, joint limits, and reachability constraints at finitely many points $\{\gamma(s_j)\}_{j=1}^{100}$.
(In a leader-follower parameterization, the leader arm's joint limits are imposed directly as bounding box constraints on the control points.)
The start and end points are also fixed.

For Riemannian metrics defined in C-space, our cost function is an approximation of the Riemannian path energy \eqref{eqn:energy}.
We use energy instead of path length \eqref{eqn:arc-length} as it encourages equal spacing of the control points, making the pointwise constraints more reliable.
(Path length has other negative properties, like being non-smooth if adjacent control points are equal.)
For control points $P$, we compute
\begin{equation}
    E(P)=\lambda\sum_{i=1}^N \Transpose{(\varphi(P_i)-\varphi(P_{i-1}))}\MetricAt{\varphi(\bar P_i)}(\varphi(P_i)-\varphi(P_{i-1})),
\end{equation}
where $\bar P_i=\frac{P_{i}+P_{i-1}}{2}$ is the midpoint in parameterized space.
Note that we lift to a chord in C-space and evaluate the metric there, rather than pulling it back to the parameterized space.
Computing the metric in chart coordinates requires the Jacobian, so gradient-based optimization would need second derivatives of the parameterization.
$\lambda$ is a scale factor, set so that the initial guess has a cost of $1$ in order to ensure the solver's optimality tolerance is consistent across trials and choice of metric.
For the flat metric, the cost function is simply the scaled path energy in parameterized space $E(P)=\lambda\sum_{i=1}^N \|P_i-P_{i-1}\|^2$.

We obtain an initial guess as a trajectory $\gamma_0:[0,1]\to\mathcal{U}$ in the parameterized space that satisfies all constraints, and time-parameterize it by arc length in the flat metric.
The initial value of $P_i$ is set to $\gamma_0(i/N)$, yielding an initial B-spline that approximately fits $\gamma_0$.
This initial guess is often still feasible, and if not, it may still be close enough to the feasible set that the optimizer can recover.
Trajectory timing is determined after solving the optimization problem using TOPP-RA \cite{pham2018new}.
\section{Experiments}
\label{sec:experiments}

\looseness=-1
We evaluate our approach on bimanual manipulation problems under a loop-closure constraint, where two 7-DoF Franka arms hold a book rigidly and move it between the compartments of a shelf (Figure~\ref{fig:teaser}).
The book starts in one of three shelf compartments, bottom (B), middle (M) and top (T), and has to reach another, which gives six problems in total.
For the implicit representation, we follow the projection-based approach of CBiRRT~\cite{berenson2009manipulation}, and for the explicit representation, we follow the analytic inverse kinematics parameterization of~\cite{cohn2024constrained}, whose differential is available in closed form~\cite{cohn2026planning}.
We consider two ambient metrics, the constant Euclidean metric and the configuration-dependent kinetic energy metric derived from the mass matrix of the two arms, so that moving the heavier joints costs more under the induced geometry.
We also include a flat objective on the parameter domain, which ignores the induced geometry and distorts distances once a path is lifted back to the original space~\cite{garg2025planning}.
All experiments minimize the induced Riemannian length in~\eqref{eqn:arc-length} under the given metric, using the midpoint approximation~\eqref{eqn:midpoint-distance}.

We use an OMPL implementation of sampling-based planner G-RRT*~\cite{kyaw2024greedy}, together with the trajectory optimizer of Section~\ref{subsec:planning-trajopt}.
G-RRT* runs anytime under a budget of 10\,s, using the matrix-valued lower bound heuristic of~\cite{kyaw2026direct} for informed sampling, with VAMP's collision checking backend~\cite{thomason2024motions}, and its paths are post-processed with randomized shortcutting and L-BFGS smoothing.
The trajectory optimizer instead solves a sequential quadratic program with SNOPT through Drake.
All results are reported after applying arclength reparameterization under the respective metric.
For each problem, we run 50 trials with different pseudorandom seeds and report the median.

\subsection{Empirical Studies}

\looseness=-1
We first validate Theorem~\ref{thm:equivalence} empirically.
For the solution paths found for each of the six problems under the three metrics, we measure the induced length of each path twice, once with the pullback metric in the parameter domain and once with the induced metric in the implicit representation.
The explicit length uses the full joint velocity produced by a step in the parameter domain, while the implicit length uses only its component in the tangent space of the constraint manifold, where the induced metric is defined.
A difference between the two therefore would indicate a velocity that leaves the constraint manifold.
We also sample 10,000 random configurations in the parameter domain and compute the eigenvalue ratio of the induced metric at each one since the solution paths cover only part of $\Cfg$.

\looseness=-1
Table~\ref{tab:equivalence} shows that the two lengths match under both metrics.
Every joint velocity the parameterization produces therefore lies in the tangent space of the constraint manifold, and shortening a path in the parameter domain shortens the physical motion.
The last row shows that the parameterization is far from an isometry even under the Euclidean ambient metric, and the kinetic energy metric raises the ratio further.
A planner that treats the parameterization domain as flat therefore optimizes the wrong functional under either ambient metric.
We measure what this costs in Section~\ref{subsec:exp-planning}.

\begin{figure*}[!t]
\centering
\vspace{-2mm}
\includegraphics[width=0.9\textwidth]{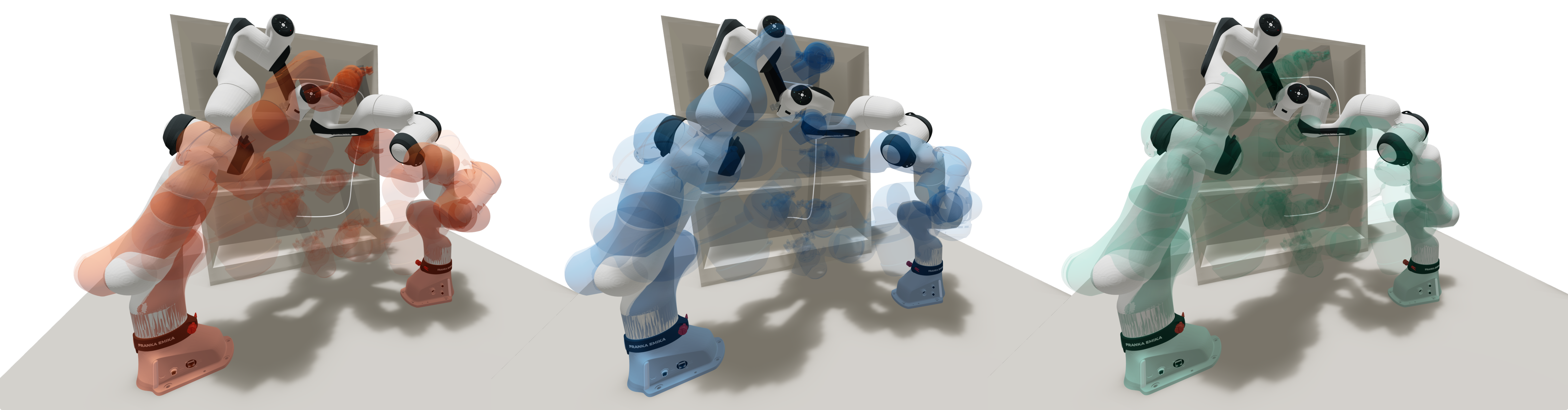}
\caption{
Motions produced by the trajectory optimizer in the explicit representation for the B$\rightarrow$T problem of the bimanual experiment described in Section~\ref{sec:experiments}.
Each panel overlays the two arms at intermediate configurations, with the white curve tracing the book's path.
The flat metric measures length in the coordinates of the parameter domain and gives the longest motion of the three (orange, left), while the induced Euclidean metric on the constraint manifold takes the shortest route (blue, middle).
The kinetic-energy metric weights the motion by the mass matrix, so the heavier joints of both arms scarcely move (green, right).
}
\label{fig:bimanual-plans}
\vspace{-2mm}
\end{figure*}

\begin{table}[!t]
\centering
\caption{
Induced Riemannian length of the same solution paths, measured once in each representation and summed over the six problems of Section~\ref{sec:experiments}.
Rows name the metric the planner minimized and each column pair the metric we measure under.
The last row reports the median eigenvalue ratio of the pullback metric over 10,000 sampled configurations.
}
\label{tab:equivalence}
\footnotesize
\setlength{\tabcolsep}{4pt}
\begin{tabular}{lrrrr}
\toprule
& \multicolumn{4}{c}{Measured under} \\
\cmidrule(lr){2-5}
& \multicolumn{2}{c}{Euclidean} & \multicolumn{2}{c}{kinetic-energy} \\
\cmidrule(lr){2-3} \cmidrule(lr){4-5}
Planned under & explicit & implicit & explicit & implicit \\
\midrule
flat & 28.80 & 28.80 & 17.70 & 17.70 \\
Euclidean & \textbf{27.25} & \textbf{27.25} & 18.01 & 18.01 \\
kinetic-energy & 32.93 & 32.93 & \textbf{15.77} & \textbf{15.77} \\
\cmidrule(lr){1-5}
Total & 88.97 & 88.97 & 51.49 & 51.49 \\
\midrule
\midrule
Eigenvalue ratio & \multicolumn{2}{c}{56} & \multicolumn{2}{c}{$5.2 \times 10^{2}$} \\
\bottomrule
\end{tabular}
\vspace{-3mm}
\end{table}

\begin{table*}[!t]
\centering
\caption{
Planning and trajectory optimization results over 50 trials on the six bimanual problems described in Section~\ref{sec:experiments}.
Rows correspond to the metric that each solver minimizes, grouped by representation and by solver, and columns to the six problems under the induced Euclidean length (left) and the induced kinetic-energy length (right).
We report median induced Riemannian lengths, with planner motions taken after smoothing and the optimizer seeded from them.
}
\label{tab:planning}
\footnotesize
\setlength{\tabcolsep}{2.4pt}
\renewcommand{\arraystretch}{1.1}
\begin{tabular}{lrrrrrr@{\hspace{1.2em}}rrrrrr}
\toprule
& \multicolumn{6}{c}{Induced Euclidean length} & \multicolumn{6}{c}{Induced kinetic-energy length} \\
\cmidrule(lr){2-7} \cmidrule(lr){8-13}
\hspace{1.1em}Metric & B$\rightarrow$M & B$\rightarrow$T & M$\rightarrow$B & M$\rightarrow$T & T$\rightarrow$B & T$\rightarrow$M & B$\rightarrow$M & B$\rightarrow$T & M$\rightarrow$B & M$\rightarrow$T & T$\rightarrow$B & T$\rightarrow$M \\
\midrule
\multicolumn{13}{l}{\textbf{Implicit representation}} \\
\rowcolor{gray!15}[\tabcolsep][\tabcolsep]
\multicolumn{13}{l}{\textit{Sampling-based planner, 10\,s anytime budget}} \\
\hspace{1.1em}Euclidean & \textbf{4.08} & \textbf{5.79} & \textbf{4.06} & \textbf{4.70} & \textbf{5.78} & \textbf{4.70} & \textbf{2.63} & \textbf{3.48} & 2.66 & \textbf{2.62} & 3.55 & 2.71 \\
\hspace{1.1em}kinetic-energy & 4.29 & 6.05 & 4.27 & 5.05 & 6.02 & 4.96 & 2.67 & 3.49 & \textbf{2.65} & 2.66 & \textbf{3.47} & \textbf{2.62} \\
\midrule\midrule
\multicolumn{13}{l}{\textbf{Explicit representation}} \\
\rowcolor{gray!15}[\tabcolsep][\tabcolsep]
\multicolumn{13}{l}{\textit{Sampling-based planner, 10\,s anytime budget}} \\
\hspace{1.1em}flat & 4.48 & 6.22 & 4.54 & 5.01 & 6.38 & 4.93 & 2.78 & 3.81 & 2.82 & 2.82 & 3.87 & 2.76 \\
\hspace{1.1em}Euclidean & \textbf{3.90} & \textbf{5.85} & \textbf{3.90} & \textbf{4.60} & \textbf{5.88} & \textbf{4.58} & 2.66 & 3.68 & 2.67 & 2.74 & 3.70 & 2.72 \\
\hspace{1.1em}kinetic-energy & 4.76 & 6.95 & 4.73 & 5.97 & 6.94 & 5.94 & \textbf{2.54} & \textbf{3.60} & \textbf{2.53} & \textbf{2.44} & \textbf{3.66} & \textbf{2.37} \\
\midrule
\rowcolor{gray!15}[\tabcolsep][\tabcolsep]
\multicolumn{13}{l}{\textit{Trajectory optimizer, seeded from the explicit motions}} \\
\hspace{1.1em}flat & 4.76 & 7.01 & 4.76 & 5.46 & 7.01 & 5.46 & 2.83 & 3.98 & 2.82 & 2.93 & 3.98 & 2.93 \\
\hspace{1.1em}Euclidean & \textbf{3.57} & \textbf{5.21} & \textbf{3.57} & \textbf{4.18} & \textbf{5.21} & \textbf{4.18} & 2.48 & 3.27 & 2.48 & 2.49 & 3.27 & 2.48 \\
\hspace{1.1em}kinetic-energy & 4.66 & 6.43 & 4.66 & 5.62 & 6.42 & 5.62 & \textbf{2.05} & \textbf{2.64} & \textbf{2.05} & \textbf{1.63} & \textbf{2.64} & \textbf{1.63} \\
\midrule
\rowcolor{gray!15}[\tabcolsep][\tabcolsep]
\multicolumn{13}{l}{\textit{Trajectory optimizer, seeded from the implicit motions}} \\
\hspace{1.1em}Euclidean & \textbf{3.57} & \textbf{5.21} & \textbf{3.57} & \textbf{4.18} & \textbf{5.21} & \textbf{4.18} & 2.48 & 3.26 & 2.48 & 2.48 & 3.27 & 2.48 \\
\hspace{1.1em}kinetic-energy & 4.65 & 6.42 & 4.65 & 5.62 & 6.42 & 5.62 & \textbf{2.05} & \textbf{2.64} & \textbf{2.05} & \textbf{1.63} & \textbf{2.64} & \textbf{1.63} \\
\bottomrule
\end{tabular}
\vspace{-3mm}
\end{table*}

\begin{table}[!t]
\centering
\caption{
Median solve time over 50 trials for the trajectory optimizer in the explicit representation.
Rows correspond to the metric that produced the initial guess and columns to the metric that the optimizer minimizes.
}
\label{tab:compatibility}
\footnotesize
\setlength{\tabcolsep}{3pt}
\renewcommand{\arraystretch}{1.1}
\begin{tabular}{lrrr}
\toprule
& \multicolumn{3}{c}{Time (s) $\downarrow$} \\
\cmidrule(lr){2-4}
Initial guess & flat & Euclidean & kinetic-energy \\
\midrule
flat & \textbf{9.2} & 55.2 & 92.8 \\
Euclidean & 9.9 & \textbf{41.8} & 80.5 \\
kinetic-energy & 10.9 & 44.3 & \textbf{78.8} \\
\bottomrule
\end{tabular}
\vspace{-4mm}
\end{table}

\subsection{Constrained Bimanual Manipulation Problems}
\label{subsec:exp-planning}

We run the sampling-based planner under both implicit and explicit representations of the constraint manifold,
and then seed the trajectory optimizer with every solution, which we restrict it to only run in the parameter domain of the explicit representation as described in Section~\ref{subsec:planning-trajopt}.
We measure every resulting path under both induced lengths and report them in Table~\ref{tab:planning}, so that we can compare all of the metrics under each.

We find that every planner variant reaches the lowest path cost under the respective metric that it plans with, on every problem and in both representations.
The flat metric produces the longest motions of the three under both induced lengths, since it measures a path in the coordinates of the parameter domain and not on the constraint manifold itself.
Pulling an ambient metric back into the parameter domain corrects this distortion, and the planner then shortens the physical motion itself (Figure~\ref{fig:bimanual-plans}).
The two induced metrics rank the same motions differently, so the ambient metric we choose decides which motion the planner returns.

\looseness=-1
The choice of metric also affects how quickly the trajectory optimizer converges.
We run the optimizer under every metric from every initial guess, and Table~\ref{tab:compatibility} reports the median solve time.
The optimizer converges fastest on most problems when its initial guess comes from the planner that used the same metric.
The guess from the flat metric is the slowest of the three under both induced metrics on every problem, since it starts the optimizer from a path that is already long under the length being minimized.
The optimized motions themselves are the same regardless of the initial guess, and only the solve time differs.

\subsection{Execution on the Real Bimanual Setup}

We run the planned motions on the two arms of Figure~\ref{fig:teaser}, which hold the book with the grasp that our constraint map describes.
To obtain the plans, we seed trajectory optimization with a Euclidean-metric plan obtained from G-RRT*.
Then, each motion is resampled into waypoints that are uniform in induced arclength and time-parameterized with TOPP-RA under the Frankas' velocity and acceleration limits~\cite{pham2018new}.
The left arm tracks the trajectory under joint position control while the right arm runs under compliance control, as internal forces may build up due to differences between the arms' idealized kinematics and the real kinematics.
We recorded the commanded and the measured joint angles together with the measured joint velocities  at 1\,kHz, and computed the trajectory timing, joint-space path length, and kinetic energy integrated over time. 

Each task runs five times under each objective, and Table~\ref{tab:hardware} reports the median for each evaluation metric. 
Similar to our simulation experiments, our real experiments show that each induced metric optimizes the measurement analogue of its path cost. 
The Euclidean metric obtains the shortest joint-space path lengths, while also generating the fastest plans; kinetic-energy plans taking 17$\%$ to 30$\%$ longer than Euclidean.
Nonetheless, the kinetic-energy metric generates the plans with the lowest integrated kinetic energy for every task.

\begin{table*}[!t]
\centering
\caption{Planning and trajectory optimization results on our real experimental setup. Each motion runs for five trials, and we report the median of trajectory duration, measured joint-space path length, and integrated kinetic energy.}
\label{tab:hardware}
\scriptsize
\setlength{\tabcolsep}{1.5pt}
\renewcommand{\arraystretch}{1.1}
\begin{tabular}{l rrrrrr @{\hspace{0.6em}} rrrrrr @{\hspace{0.6em}} rrrrrr}
\toprule
& \multicolumn{6}{c}{Executed duration [s]} & \multicolumn{6}{c}{Joint-space path length [rad]} & \multicolumn{6}{c}{Integrated kinetic energy [J\,s]} \\
\cmidrule(lr){2-7} \cmidrule(lr){8-13} \cmidrule(lr){14-19}
Metric & B$\to$M & B$\to$T & M$\to$B & M$\to$T & T$\to$B & T$\to$M & B$\to$M & B$\to$T & M$\to$B & M$\to$T & T$\to$B & T$\to$M & B$\to$M & B$\to$T & M$\to$B & M$\to$T & T$\to$B & T$\to$M \\
\midrule
flat & 14.11 & 20.97 & 14.12 & 12.89 & 20.97 & 12.94 & 6.64 & 9.56 & 6.64 & 7.33 & 9.56 & 7.33 & 0.281 & 0.382 & 0.281 & 0.295 & 0.382 & 0.295 \\
Euclidean & \textbf{12.19} & \textbf{12.63} & \textbf{12.19} & \textbf{10.20} & \textbf{12.60} & \textbf{10.24} & \textbf{5.03} & \textbf{7.38} & \textbf{5.04} & \textbf{5.91} & \textbf{7.38} & \textbf{5.91} & 0.266 & 0.443 & 0.266 & 0.309 & 0.444 & 0.309 \\
kinetic-energy & 14.33 & 14.81 & 14.33 & 13.23 & 14.75 & 13.25 & 6.60 & 9.06 & 6.60 & 7.94 & 9.06 & 7.93 & \textbf{0.158} & \textbf{0.320} & \textbf{0.158} & \textbf{0.124} & \textbf{0.320} & \textbf{0.124} \\
\bottomrule
\end{tabular}
\vspace{-4mm}
\end{table*}
\vspace{-2mm}
\section{Conclusion}
\label{sec:conclusion}

This work treats motion planning with constraints as a search over a Riemannian submanifold that inherits its metric from the ambient configuration space.
Our main result is that both the implicit and explicit representations induce the same metric, related by an invertible change of basis, for any ambient Riemannian metric.
The two representations therefore describe one geometry, which lets a planner choose its metric and its representation independently.
We further show how a sampling-based planner and a trajectory optimizer use this induced metric without any change to the search strategies on which they are built, so that both minimize the same Riemannian length functional on the constraint submanifold.
Across six bimanual manipulation problems with two Franka arms under a loop-closure constraint, each planner returns the shortest motion under the metric that it plans with in both representations, while a planner that treats the parameter domain as flat gives longer motions under both induced Riemannian lengths.

The retraction we use projects along the Euclidean normal, which makes it second order under the Euclidean metric and first order under a general one.
We leave a metric-weighted projection, which would recover second-order accuracy for any ambient metric, as future work.
We are also interested in trajectory optimization directly in the implicit representation~\cite{bordalba2022direct} under induced metrics, where the decision variables would lie on the level set rather than in a parameter domain.

\section*{Acknowledgements}
\looseness=-1
This work was supported in part by the Canada Research Chairs Program, the Natural Sciences and Engineering Research Council of Canada (NSERC) through Discovery Grant no. RGPIN-2023-05036, the National Science Foundation Graduate Research Fellowship Program under Grant No. 2141064, and the Secretaría de Ciencia, Humanidades, Tecnología e Innovación (SECIHTI), Mexico, under CVU number 1066015. Any opinions, findings, and conclusions or recommendations expressed in this material are those of the author(s) and do not necessarily reflect the views of the National Science Foundation.

\bibliographystyle{IEEEtran}
\bibliography{IEEEabrv, main}

\end{document}